\documentclass[10pt,twocolumn,letterpaper]{article}

\usepackage{cvpr}
\usepackage{times}
\usepackage{epsfig}
\usepackage{graphicx}
\usepackage{amsmath}
\usepackage{amsthm}
\usepackage{amssymb}

\usepackage[pagebackref=true,breaklinks=true,letterpaper=true,colorlinks,bookmarks=false]{hyperref}

\cvprfinalcopy % *** Uncomment this line for the final submission

\def\cvprPaperID{5719} % *** Enter the CVPR Paper ID here

\usepackage{hyperref}
\usepackage{url}
\usepackage{amsmath}
\usepackage{graphicx}
\usepackage{paralist}
\usepackage{amssymb}
\usepackage{mathtools}
\DeclareMathOperator*{\argmin}{arg\,min}
\DeclareMathOperator*{\rand}{rand}
\newtheorem{mydef}{Definition}

\newtheorem{proposition}{Proposition}

\renewcommand{\labelenumii}{\theenumii}
\renewcommand{\theenumii}{\theenumi.\arabic{enumii}.}

\ifcvprfinal\fi

\usepackage{xspace}
\usepackage{color}

\newcommand{\mypartitle}[2][2.5]{\vspace*{-#1 ex}~\\{\noindent {\bf #2}}}

\graphicspath{{IMAGES/}}

\begin{document}

\title{Theoretical Guarantees of Deep Embedding Losses Under Label Noise}

\author{
Nam Le$^{1,2}$ \qquad Jean-Marc Odobez$^{1,2}$ \\
$^1$ Idiap Research Institute, Martigny, Switzerland\\
$^2$ \'{E}cole Polytechnique F\'{e}d\'{e}ral de Lausanne, Switzerland \\
{\tt\small \{nle,odobez\}@idiap.ch}
}

\maketitle

\begin{abstract}
Collecting labeled data to train deep neural networks is costly and even impractical for many tasks. Thus, research effort has been focused in automatically curated datasets or unsupervised and weakly supervised learning. 
The common problem in these directions is learning with unreliable label information. In this paper, we address the tolerance of deep embedding learning losses against label noise, i.e. when the observed labels are different from the true labels. Specifically, we provide the sufficient conditions to achieve theoretical guarantees for the 2 common loss functions: marginal loss and triplet loss. From these theoretical results, we can estimate how sampling strategies and initialization can affect the level of resistance against label noise.
The analysis also helps providing more effective guidelines in unsupervised and weakly supervised deep embedding learning. 
\end{abstract}

\section{Introduction}

Embedding learning methods aim at learning a parametrized mapping function from a labeled
set of samples to a metric space, in which samples with the same labels are
close together and samples with different labels are far apart. 
To learn such an embedding with deep neural networks, the most common losses are contrastive loss~\cite{hadsell2006dimensionality} and triplet loss~\cite{Schroff2015}, which optimized based on the distances between pairs or triplets of samples respectively.
Because these losses only require the information whether 2 samples have the same or different labels, they have potential for learning with uncertainty in sample label information. 
In fact, there have been an increasing number of works that use embedding losses for unsupervised tasks by inferring the pair-wise label relationship using other sources of information~\cite{iscen2018mining,jansen2018unsupervised,radenovic2016cnn,wang2015unsupervised,yang2016joint}.
As the inferred label information is not reliable, they can be misleading and will subvert the model during training. 
This leads to our research question:

\emph{Given a dataset with unreliable labels, what are the guarantees when one learns an embedding using triplet loss or contrastive loss?}

This question is becoming more important as there are more datasets where labels are no longer curated by human but by internet queries~\cite{parkhi15deep,song2016deep,yi2014learning}, crossmodal supervision and transfer learning~\cite{le2017improving,Nagrani17}, associated social information~\cite{lee2017large,mahajan2018exploring}, or data mining~\cite{iscen2018mining,wang2017transitive}.

To answer the question, we conducted a theoretical analysis on two different types of embedding losses: marginal loss (a generalized contrastive loss)~\cite{manmatha2017sampling} and triplet loss under label noise from the empirical risk minimization perspective. 
Under this perspective, a loss function is said to be tolerant to label noise of rate $p$ if the minimizer of the empirical risk in the noise free condition is also the minimizer of the empirical risk under noise.
In our analytical results, we have proved the sufficient conditions so that:
\begin{compactitem}
\item Minimizing triplet loss produces the same global optima under label noise if the uniform noise rate is $p < 1 - \sqrt{1 - \frac{1}{\eta}} - \epsilon(K)$.

\item Minimizing marginal loss produces the same local optima under label noise if the uniform noise rate is $p < 1 - \sqrt{1 - \gamma}  - \epsilon(K)$.

%\item Triplet loss can be equally or more tolerant to uniform label noise than marginal loss.
\end{compactitem}
in which, $\epsilon$ denotes small constants depending on the number of classes, $\eta$, $\gamma$ depend on the sampling scheme and the loss parameters of each loss function. 
These theoretical results imply that learning embeddings under label noise is heavily influenced by the sampling strategy and marginal loss is robust to label noise only when the initialization is sufficiently good.
%
%We provide further argument in how sampling methods affect these $\eta$ parameters.
We have conducted experiments on standard vision datasets to demonstrate how the two embedding losses can be robust to label noise in practice and how different sampling strategies and initialization can affect the guarantees of tolerance.

\section{Related work}

While embedding losses under label noise have not been studied before, there has been a vast literature in analyzing label noise for classification.
For an in-depth introduction to label noise and a comprehensive analysis of traditional algorithms, we refer the readers to the survey of~\cite{frenay2014comprehensive}.

In the context of deep learning, most effort has been dedicated to improve training networks under label noise.
One major direction is to approximate a model of noise to improve training. There are a few examples of this direction. In ~\cite{krause2016unreasonable}, the authors use active learning to select clean data from noisy training set, in ~\cite{xiao2015learning}, there are multiple iterations of training a model, formulating the noise, and retraining, and in ~\cite{liu2016classification} the estimation of noisy labels is used to reweigh the training samples.
Another direction is to improve the networks directly to make them robust to label noise. 
For example, one can add a noise adapting layer to correct the network for the latent noise in training datasets
~\cite{sukhbaatar2014training} or augment a standard deep network with a softmax layer that models the label noise statistics~\cite{jindal2016learning}.

While all the above methods focus on changing the learning model or strategy, there is another interesting body of works in analyzing the loss functions used to train the models~\cite{drory2018resistance,ghosh2017robust,natarajan2013learning}. Here we want to highlight one such work presented in~\cite{ghosh2017robust}. In this work, the authors introduced the notion of symmetric loss functions and proved that such symmetric losses are tolerant to label noise. From the theoretical analysis, they have shown mean absolute error as a more robust alternative for cross entropy loss in training classification deep neural networks.

Within this literature, our paper can be viewed as a counterpart of~\cite{ghosh2017robust} for embedding losses.  In our work, we not only explore how the per sample label noise affects the pair-wise and triplet-wise labels but also provide further analysis on the impact of sampling and initialization, which are integral parts of learning embeddings.

\section{Preliminary}
We will recall the losses used for deep embedding learning and then define the scope of label noise to be used in subsequent sections.
\subsection{Deep embedding learning}
Given a labeled training set of $\{(x_i, y_i)\}$, in which $x_i \in \mathbb{R}^D, y_i \in \{1, 2, .., K\}$, we define an embedding function as  a parameterized $f(x; \theta) \in \mathbb{R}^d$, which maps an instance $x$ into a $d$-dimensional Euclidean space. Additionally, this embedding is constrained to live on the $d$-dimensional hypersphere, \textit{i.e.} $||f(x)||_2 = 1$. Within the hypersphere, the distance between 2 projected instances is simply the Euclidean distance:
\begin{equation}
d(f(x_i; \theta), f(x_j; \theta)) = ||f(x_i; \theta) - f(x_j; \theta) ||_2
\end{equation}

In this new embedding space, we want the intra-class distances $d(f(x_i; \theta), f(x_j; \theta))$, $\forall x_i, x_j / y_i = y_j$ to be minimized and the inter-class distances $d(f(x_i; \theta), f(x_j; \theta))$, $\forall x_i, x_j / y_i \neq y_j$ to be maximized.
%
%A major advantage of embedding learning is that the projection $f$ is class independent. At test time, we can expect examples from a different class still satisfy the embedding goals. This makes embedding learning suitable for verification and clustering tasks.
%
For shorthand, we will simply use $d_{ij}$ to replace $d(f(x_i; \theta), f(x_j; \theta))$.

\subsubsection{Marginal loss}

Marginal loss is the generalized version of contrastive loss~\cite{manmatha2017sampling,deng2017marginal}. This loss aims to separate the distances of positive pairs and negative pairs by a threshold of $\beta$ and with the margin of $\alpha$ on both sides.  To simplify the analysis, we do not consider the learnable $\beta$ parameters in~\cite{manmatha2017sampling}. For pair of samples $\{(x_i, y_i), (x_j, y_j)\}$, we can define the pair label as $t_{ij} = 1$ if $y_i = y_j$ and $t_{ij} = -1$ otherwise. Concretely, the loss for one pair is:

\begin{equation}
\label{eq:marginalloss}
l^M(x_i, x_j, t_{ij}; f) = [(d_{ij} - \beta)t_{ij} + \alpha]_+
\end{equation}

We use the shorthand notation $l_{ij}(t_{ij};\theta)$ to replace $l(x_i, x_j, t_{ij}; \theta)$.

\subsubsection{Triplet loss}
For triplet loss, we do not care about the explicit threshold but impose a relative order on a positive pair and a negative pair. 
A triplet consists of 3 data points: $(x_a, x_p, x_n)$ such that $y_a = y_p$ and $y_a \neq y_n$ and thus, we would like the 2 points $(x_a, x_p)$ to be close together and the 2 points $(x_a, x_n)$ to be further away by a margin $\alpha$ in the embedding space.
Hence, the loss for one triplet is defined as:
\begin{equation}
\label{eq:tripletloss}
l^{T}(x_a, x_p, x_n; \theta) = [d_{ap} - d_{an} + \alpha]_+
\end{equation} 

\subsubsection{Empirical risk minimization}
From the risk minimization perspective, one might aim
at optimizing the total loss over all pairs or triplets respectively.
Let $S$ be the set of all possible triplets (or pairs), the empirical risk to minimize in both cases will be:
\begin{equation}
\label{eq:empirical_risk}
R_L(S; f) = \frac{1}{|S|}\sum_{s\in S}l(s;\theta)
\end{equation}

\subsection{Label noise}
Given a sample $x_i$ with its true label $y_i$, we assume this true label can be wrongly observed with a probability $p$. Let $\hat{y}_i$ be the observed label with the following rule:

\begin{equation}
\hat{y}_i=\left\{
\begin{array}{c l}	
     y_i & \text{with prob. }  1-p_{x_i}\\
     u & \text{with prob. }  p_{{x_i}u} \qquad \forall u \neq y_i
\end{array}\right.
\end{equation}

in which $\sum_u{p_{{x_i}u}} = p_{x_i}$. If the individual noise probability is uniform and independent with the input $x_i$, we can simply write:

\begin{equation}
\label{eq:label_noise}
\hat{y}_i=\left\{
\begin{array}{c l}	
     y_i & \text{with prob. }  1-p\\
     u & \text{with prob. }  \frac{p}{K-1} \qquad \forall u \neq y_i
\end{array}\right.
\end{equation}
While the analysis can be applied on complicated distributions of noise, we assume that the label noise on the individual sample is uniform and independent of $x_i$. Thus we only take into account the sample label noise rate $p$ in Eq.~\ref{eq:label_noise}.

\subsection{Relationship between sample label noise $p$ and pair label noise $q$}

We want to compute given the sample label noise rate of $p$, what is the pair label noise rate $q$. In another word, for a pair of samples with original pair label of $t_{ij} \in \{-1, 1\}$, we want to find the chance that $t_{ij}$ is corrupted into $-t_{ij}$
\vspace{6mm}
\mypartitle{Negative case $t_{ij} = -1$}
The probability a negative pair is corrupted into a positive pair is decomposed into 2 cases:
\begin{compactitem}
\item one of the two samples changes its label, and the new label is the same with the other one: $2p\frac{(1-p)}{K-1}$
\item both samples' labels change into 2 different labels, and both labels are the same: $\frac{p^2(K-2)}{(K-1)^2}$
\end{compactitem}
Hence, in this negative case:
\begin{equation}
\label{eq:negative_prob}
q_{-1} = 2p\frac{(1-p)}{K-1} + \frac{p^2(K-2)}{(K-1)^2}
\end{equation}
%\vspace{6mm}
%
\mypartitle{Positive case $t_{ij} = 1$}
The probability a positive pair is corrupted into a negative pair is decomposed into when:
\begin{compactitem}
\item one of the two samples changes its label any different label: $2p(1-p)$
\item both samples change into different labels: $p^2\big(1-\frac{2}{K-1}\big)$
\end{compactitem}
In this positive case:
\begin{equation}
\label{eq:positive_prob}
q_{1} = 2p(1-p) + p^2\Big(1-\frac{2}{K-1}\Big)
\end{equation}

\section{Triplet loss under label noise}
\label{sec:triplet}
A triplet is chosen based on the observed labels, $\hat{y}_a$, $\hat{y}_p$, and $\hat{y}_n$. However, as these labels can be noisy, the true labels can be one of 3 following cases:
\begin{compactitem}
\item $y_a = y_p = y_n$
\item $y_a \neq y_p \neq y_n$
\item $y_a \neq y_p$ and $y_a = y_n$
\end{compactitem}
The determine the condition for triplet loss to be robust to label noise, we first decompose it into a combination of auxiliary pair-wise losses and consider the unhinged triplet loss.

\subsection{Auxiliary pair-wise and unhinged triplet loss} 

\begin{mydef}
\label{def:auxiliary}
We define an auxiliary pair-wise loss $l^A$ as:
\begin{equation}
l^{A}(x_i, x_j, t_{ij}; \theta)=\left\{
\begin{array}{c l}	
     d_{ij}t_{ij} & \text{if }  t_{ij} = 1\\
     d_{max} + d_{ij}t_{ij} & \text{if }  t_{ij} = -1
\end{array}\right.
\end{equation}
in which $d_{max} = 2$ is the maximum distance between 2 points on the hypersphere.
Note the property that: %$\sum_{t{ij}}l^{A}_{ij}(t_{ij}; \theta) = d_{max}, , or 
\begin{equation}
l^{A}_{ij}(-t_{ij}; \theta) = d_{max} - l^{A}_{ij}(t_{ij}; \theta), \forall i, j
\end{equation}
\end{mydef}

\begin{mydef}
We define the label-dependent weighted version of the auxiliary loss as when each pair $(x_i, x_j, t_{ij})$ is weighted differently by $w_{t_{ij}}$, in which the weight only depends on the pair label. Under noise, when a pair changes its pair label from $t_{ij}$ into $-t_{ij}$, its weight only changes from $w_{t_{ij}}$ into $w_{-t_{ij}}$.

Hence, under label noise, the risk to minimize per pair $(x_i, x_j, t_{ij})$ is:
\begin{equation}
\label{eq:sample-risk}
\begin{split}
w_{\hat{t}_{ij}}\hat{l}^A(x_i, x_j, t_{ij}; \theta) =&(1 - q_{t_{ij}}) w_{t_{ij}} l^A_{ij}(t_{ij};\theta) \\
& + q_{t_{ij}} w_{-t_{ij}} l^A_{ij}(-t_{ij};\theta) \\
=\Big(1-q_{t_{ij}} - q_{t_{ij}}&\frac{w_{-t_{ij}}}{w_{t_{ij}}}\Big)w_{t_{ij}}l^A(x_i, x_j, t_{ij}, \theta) \\
+w_{-t_{ij}}&q_{t_{ij}}d_{max} 
\end{split}
\end{equation}
\end{mydef}

\noindent In Eq~\ref{eq:sample-risk}, we have used the fact that $l^{A}_{ij}(-t_{ij}; \theta) = d_{max} - l^{A}_{ij}(t_{ij}; \theta)$ in Def.~\ref{def:auxiliary}. Here one can observe that the risk under noise for one pair $w_{\hat{t}_{ij}}\hat{l}^A(x_i, x_j, t_{ij}; \theta)$ is actually just the scaled clean risk with a constant offset. Therefore, if the clean risk is minimized for that pair, the noisy risk is also minimized. In the next steps, we will prove the same thing for the risk under noise over the whole training set of pairs. 

\begin{mydef}
For a given triplet of $x_a, x_p, x_n / y_a = y_p \wedge y_a \neq y_n$, we define the unhinged triplet loss $l_U$, which can be decomposed into auxiliary pair-wise loss, as:
\begin{equation}
\label{eq:unhinged}
\begin{split}
l^U(x_a, x_p, x_n; \theta) &= d_{max} + d_{ap} - d_{an} + \alpha \\
&=l^A_{ap}(t_{ap}, \theta)+ l^A_{an}(t_{an}, \theta) + \alpha
\end{split}
\end{equation}
\end{mydef}

\begin{mydef}
We define a 1-1 sampling scheme for triplet loss as when for a given positive pair, out of all possible negative pairs of the anchor, only 1 negative pair is chosen.
\end{mydef}

\begin{proposition}
\label{prop:main_triplet}
A minimizer $\theta^{*}$ of the empirical risk with unhinged triplet loss
in the noise free condition $R_{l^U}(S, \theta)$ 
 is also the minimizer
of the empirical risk with unhinged triplet loss
under noise $\hat{R}_{l^U}(S, \theta)$ if:
\begin{compactenum}
\item A 1-1 sampling scheme is used.
%\item The sample label noise rate $p < 0.5$.
\item $\theta^*$ is the minimizer of 2 summations $\mathcal{S}^+$ and $\mathcal{S}^-$:
\begin{compactitem}
\item $\mathcal{S}^+ = \sum_{ij}{l^A(x_i, x_j, t_{ij}, \theta)}$, $\forall (i,j) / t_{ij} = 1$.
\item $\mathcal{S}^- = \sum_{ij}{l^A(x_i, x_j, t_{ij}, \theta)}$, $\forall (i,j) / t_{ij} = -1$.
\end{compactitem}
\end{compactenum}
\end{proposition}

\begin{proof}
\noindent As $l^U$ can be decomposed as a linear combination of $l^A$, the unhinged empirical risk over all possible triplets can be rewritten as:
\begin{equation}
R_{l^U}(S, \theta) = \frac{1}{Z}\sum_{ij}N^{ij} l^A(x_i, x_j, t_{ij}, \theta)
\end{equation}

\noindent In which, $Z$ is the normalizing number and $N_{ij}$ is the weight as each pair can be chosen multiple times in triplet loss. Assuming that there are uniformly $s$ samples per every class and there are $K$ classes, then $N^{ij} = (K-1)s$ if $t_{ij} = 1$ (each positive pair can be combined with $(K-1)s$ negative pairs) and $N^{ij}=s-1$ if $t_{ij} = -1$ (each negative pair can be combined with $s-1$ positive pairs).

\noindent When a 1-1 sampling scheme is applied, each positive pair is chosen only once, while the probability that one negative pair is chosen is approximately $\frac{1}{K}$ if each class has a uniform number of samples. After the sampling scheme, we have the weighted empirical risk:
\begin{equation}
R_{l^U}(S, \theta) = \frac{1}{Z}\sum_{ij}w_{t_{ij}} l^A(x_i, x_j, t_{ij}, \theta)
\end{equation}
\noindent where $w_{t_{ij}}$ is the weight associated to the probability each pair is chosen:
\begin{equation}
\label{eq:pair_weight}
w_{t_{ij}}=\left\{
\begin{array}{c l}	
     1 & \text{if }  t_{ij} = 1\\
     \frac{1}{K} & \text{if }  t_{ij} = -1
\end{array}\right.
\end{equation}

\noindent The weighted empirical risk under noise will then be:
\begin{equation}
\begin{split}
R_{l^U}(S, \theta) = &\frac{1}{Z}\sum_{ij}w_{\hat{t}_{ij}} \hat{l}^A(x_i, x_j, t_{ij}, \theta) \\
= &\frac{1}{Z}\Big[\sum_{ij}w_{t_{ij}}(1-q_{t_{ij}})l^A(x_i, x_j, t_{ij}, \theta) \\
&+ \sum_{ij}w_{-t_{ij}}q_{t_{ij}}l^A(x_i, x_j, -t_{ij}, \theta)\Big]
\end{split}
\end{equation}

Using the result from the noisy risk per one pair in Eq.~\ref{eq:sample-risk}, we can rewrite the empirical risk under noise as:
\begin{equation}
\label{eq:emp_risk}
\begin{split}
\hat{R}_{l^U}(S, \theta) = \frac{1}{Z}\Big[&\sum_{ij}w_{t_{ij}}(1-q_{t_{ij}})l^A(x_i, x_j, t_{ij}, \theta) \\
+ &\sum_{ij}w_{-t_{ij}}q_{t_{ij}}\big(d_{max} - l^A(x_i, x_j, t_{ij}, \theta)\big)\Big]
\\
= \frac{1}{Z}\sum_{ij}\Big[\Big(
1-&q_{t_{ij}} - q_{t_{ij}}\frac{w_{-t_{ij}}}{w_{t_{ij}}}\Big)w_{t_{ij}}l^A(x_i, x_j, t_{ij}, \theta) \\
+ &w_{-t_{ij}}q_{t_{ij}}d_{max} \Big]
\end{split}
\end{equation}

\noindent Let $\theta^*$ be the optimizer of the clean risk $R_{l^U}(S, \theta)$, which gives us:
\begin{equation}
R_{l^U}(\theta^*) - R_{l^U}(\theta) \leq 0 \qquad \forall \theta \label{eq:risk_min1}
\end{equation}

\noindent We consider the same $\theta^*$ in the noisy risk $\hat{R}_{l^U}(\theta^*) - \hat{R}_{l^U}(\theta)$, and then apply~\ref{eq:emp_risk}:
\begin{equation}
\label{eq:risk_min3}
\begin{split}
\hat{R}_{l^U}(\theta^*&) - \hat{R}_{l^U}(\theta) \\
=\frac{1}{Z}&\sum_{ij}\Big[
\Big(1-q_{t_{ij}} - q_{t_{ij}}\frac{w_{-t_{ij}}}{w_{t_{ij}}}\Big) \\
&\times\Big(w_{t_{ij}}l^A(x_i, x_j, t_{ij}, \theta^*) - w_{t_{ij}}l^A(x_i, x_j, t_{ij}, \theta)\Big)\Big]
\end{split}
\end{equation}

\noindent From the condition, we have $\theta^*$ is also the minimizer of $\mathcal{S}^+$ and $\mathcal{S}^-$, or:
\begin{equation}
\begin{split}
\sum_{ij/t_{ij} = 1}{l^A(x_i, x_j, t_{ij}, \theta^*)} &- \sum_{ij/t_{ij} = 1}{l^A(x_i, x_j, t_{ij}, \theta)}\leq 0 \\
\sum_{ij/t_{ij} = -1}{l^A(x_i, x_j, t_{ij}, \theta^*)} &- \sum_{ij/t_{ij} = -1}{l^A(x_i, x_j, t_{ij}, \theta)}\leq 0
\end{split}
\end{equation}

\noindent Using this fact to upper bound Eq.~\ref{eq:risk_min3}\footnote{More explanation is provided in the supplementary} we can come to:
\begin{equation}
\label{eq:risk_min5}
\begin{split}
\hat{R}_{l^U}(S, \theta^*) - \hat{R}_{l^U}(S, \theta) &\leq \\
\min_{t_{ij}}\Big(1-q_{t_{ij}} - q_{t_{ij}}&\frac{w_{-t_{ij}}}{w_{t_{ij}}}\Big)\Big(R_{l^U}(S, \theta^*) - R_{l^U}(S, \theta)\Big)
\end{split}
\end{equation}

\noindent This upper bound in Eq.~\ref{eq:risk_min5} is reached when the following condition is satisfied:
\begin{equation}
\label{eq:min_p}
Q = \min_{t_{ij}}\Big(1-q_{t_{ij}} - q_{t_{ij}}\frac{w_{-t_{ij}}}{w_{t_{ij}}}\Big) \geq 0
\end{equation}

\noindent From \ref{eq:risk_min1} and \ref{eq:risk_min5}, we have:
\begin{equation}
\hat{R}_{l^U}(\theta^*) - \hat{R}_{l^U}(\theta) \leq Q (R_{l^U}(\theta^*) - R_{l^U}(\theta))  \leq 0 \label{eq:risk_min}\\
\end{equation}

\noindent Hence, $\theta^*$ will also be the minimizer of the noisy risk $\hat{R}_{l^U}(S, \theta)$ if the condition Eq.~\ref{eq:min_p} is met. Using the value of $w_{t_{ij}}$ in Eq.~\ref{eq:pair_weight} , we have:
\begin{equation}
\begin{split}
1-q_{+1} - q_{+1}\frac{w_{-1}}{w_{+1}} = 1-q_{+1} - \frac{q_{+1}}{K} &\geq 0 \\
1-q_{-1} - q_{-1}\frac{w_{+1}}{w_{-1}} = 1-q_{-1} - q_{-1}K &\geq 0 
\end{split}
\end{equation}

\noindent Using the value of $q_{t_{ij}}$ in Eq.~\ref{eq:negative_prob} and~\ref{eq:positive_prob}, and let $r(K)$ be all the terms with $K$ in the denominator, we can simplify $Q \geq 0$ into:
\begin{equation}
\label{eq:min_final}
\begin{split}
1-2p+p^2 - r(K) &\geq 0
\end{split}
\end{equation}

\noindent If we assume $K$ is very large, then $r(K) \approx 0$.\footnote{We henceforth use $r(K)$ and $\epsilon(K)$ to denote small values depending on $K$}
Then Eq.~\ref{eq:min_final} is true when $p < 1 - \epsilon(K)$. This concludes the proof.
\end{proof}

A model can achieve the condition 2 that $\theta^*$ is the minimizer of $\mathcal{S}^+$ and $\mathcal{S}^-$ when on average, all the positive pairs are as close as they can be and the negative pairs are as far as they can be. 
In other words, a ideal noise free model learned with triplet loss should be sufficiently good at separating inputs into their respective clusters to guarantee that the model learned under noise will be robust to noise.
In short, optimizing unhinged triplet loss will be noise tolerant if 2 prerequisites are satisfied: a 1-1 sampling scheme is used and the model is sufficiently good over all pairs of ideal input data.%This also shows that unhinged triplet loss over the whole set has high nois tolerant rate than marginal loss over symmetric set.

\subsection{Triplet loss and semi-hard mining}

When applying triplet loss in practice, there are 2 main differences from the theoretical unhinged version: the hinge function and semi-hard triplet mining. We first consider the hinge function. By setting a threshold in choosing the triplets, it gives higher weights to harder negative pairs and lower weights to easier negative pairs with respect to the positive distance. Concretely, for $t_{ij}= -1$, $w_{ij}$ can be $\frac{\eta_{ij}}{K}$ for the harder pairs and $\frac{1}{\eta_{ij} K}$ for easier pairs, with $\eta_{ij}$ being some value greater than 1. Using this new value of $w_{ij}$ in Eq.~\ref{eq:min_p}, we can have the condition:
\begin{equation}
1 - (2p + q^2)\eta - r(K)\geq 0
\end{equation}
This gives us the new bound of sample label noise is:
\begin{equation}
p < 1 - \sqrt{1 - \frac{1}{\eta}} - \epsilon(K)
\end{equation}
with $\eta = max\{\eta_{ij}\}$

Intuitively, because noisy negative pairs have smaller distances, they are more likely to be chosen by a factor of $\eta$. This makes triplet loss less resistant to label noise also by a factor of $\eta$. Though $\eta$ cannot be computed in practice, ideally the more uniformly negative pairs are sampled, the smaller the value of $\eta$ is.

\mypartitle{$\eta$ and sampling strategies.} Due to the fact that the way negative pairs are sampled depends on the mining strategy used, we now investigate 2 different variants of semi-hard triplet mining, namely random semi-hard and fixed semi-hard, as follows:
\begin{compactitem}
\item Random semi-hard: for every positive pair, we randomly sample one negative pair so that the corresponding triplet loss is non negative. Concretely, given the positive pair $x_a, x_p$, the negative index $n^*$ is chosen as:
\begin{equation}
n^* = \rand_n\{n / d_{ap}-d_{an} + \alpha > 0\}
\end{equation}
\item Fixed semi-hard: for every positive pair, we sample the hardest negative pair so that the corresponding triplet loss is still less than $\alpha$ (ie. the hardest semi-hard negative pair). Thus, given the positive pair $x_a, x_p$, the negative index $n$ is chosen as:
\begin{equation}
n^* = \argmin_{n / d_{ap}<d_{an}}{d_{an}}
\end{equation}
\end{compactitem}

One can observe that both semi-hard mining strategies are 1-1 sampling schemes. The negative pairs sampled by fixed semi-hard mining will be more concentrated in the harder range than the negative pairs sampled by random semi-hard mining. Therefore, we can conjecture that fixed semi-hard mining will have a larger skew value $\eta$ than that of random semi-hard, or $\eta_{rand} < \eta_{fixed}$. In the experiments, we will show further how the value of $\eta$ varies based on the sampling scheme in the investigated datasets.

\section{Marginal loss under label noise}

\subsection{Label noise in marginal loss}

Marginal loss is defined based on the label of the pair $t_{ij} \in \{-1, 1\}$. Therefore, if a sample label has a noise probability $p$, a pair label will have a noise probability of $q_{x_i,x_j,t_{ij}}$. 
We assume this noise only depends on the value of the label (ie. class conditional noise), thus simplifying $q_{x_i,x_j,t_{ij}}$ into $q_{t_{ij}}$
Once 2 points $x_i$ and $x_j$ are sampled, the observed pair label $\hat{t}_{ij}$ follows:
\begin{equation}
\label{eq:margin-noise}
\hat{t}_{ij}=\left\{
\begin{array}{c l}	
     t_{ij} & \text{with prob. }  1-q_{t_{ij}}\\
     -t_{ij} & \text{with prob. }  q_{t_{ij}}
\end{array}\right.
\end{equation}

\subsection{Relation from unhinged triplet loss to unhinged marginal loss}
\begin{mydef}
Similarly to the auxiliary pair-wise loss, we define the unhinged marginal loss as follows:
\begin{equation}
l^{M}(x_i, x_j, t_{ij}, \theta)= d_{max} + (d_{ij} - \beta)t_{ij} + \alpha
\end{equation}
We can observe the similar property that:
\begin{equation}
l^{M}_{ij}(-t_{ij}; \theta) = 2d_{max} + 2\alpha - l^{M}_{ij}(t_{ij}; \theta), \forall i, j
\end{equation}
\end{mydef}

For the unhinged marginal loss, we can also define a corresponding 1-1 sampling scheme so that for each positive pair, 1 negative pair is chosen with 1 common end point as in~\cite{manmatha2017sampling}. Using this sampling scheme, one can observe that unhinged marginal loss can be combined into a translated version of unhinged triplet loss. Using the same analysis as in Section~\ref{sec:triplet}, we can show that minimizing empirical risk with unhinged marginal loss is under noise will yield the same minimizer as with minimizing empirical risk without noise.

\subsection{Marginal loss with hinge function and mining}

When the hinge function is applied, some pairs will yield 0 loss. This is similar to saying that some easy pairs are filtered out, and the more difficult pairs will be sampled more. Concretely, for a pair $(x_i, x_j, t_{ij})$:
\begin{compactitem}
\item $t_{ij} = 1$: $w_{ij} = 0$ if $d_{ij} \leq \beta - \alpha$ and $w_{ij} = \eta^{+}_{ij}$ otherwise, with $\eta^{+}_{ij} > 1.0$
\item $t_{ij} = -1$: $w_{ij} = 0$ if $d_{ij} \geq \beta + \alpha$ and $w_{ij} = \frac{\eta^{-}_{ij}}{K}$ otherwise, with $\eta^{-}_{ij} > 1.0$.
\end{compactitem}

We can see that over some subsets of pairs, $w_{ij} = 0$, therefore $\frac{\hat{w}_{ij}}{w_{ij}}$ as in Eq.~\ref{eq:risk_min3} does not exist for these pairs. To deal with these pairs, we divide the input set $\mathcal{T}$ into 3 sets: 
\begin{compactitem}
\item The set in which the pair weight is positive in the noise free case, ie. $w_{ij} = \eta_{ij}$, but is 0 under the noisy case, ie. $\hat{w}_{ij}  = 0$. This set is denoted as $\mathcal{T}_m^+$. The noisy risk for a pair $(x_i, x_j, t_{ij}) \in \mathcal{T}_m^+$ is:
\begin{equation}
\begin{split}
\label{eq:T_plus}
\hat{w}_{ij}\hat{l}_{ij}(t_{ij}, \theta) &= (1-q_{t_{ij}})w_{ij}l_{ij}(t_{ij}, \theta)
\end{split}
\end{equation}
\item The set in which the pair weight is 0 in the noise free case, ie. $w_{ij}  = 0$, but is positive under the noisy case, ie. $\hat{w}_{ij} = \eta_{ij}$. This set is denoted as $\mathcal{T}_m^-$. The noisy risk for a pair $(x_i, x_j, t_{ij}) \in \mathcal{T}_m^-$ is 
\begin{equation}
\label{eq:T_minus}
\begin{split}
\hat{w}_{ij}\hat{l}_{ij}(t_{ij}, \theta) &= q_{t_{ij}}\hat{w}_{ij}l_{ij}(-t_{ij}, \theta)
\end{split}
\end{equation}
%&= q_{t_{ij}}\hat{w}_{ij}(2d_{max} + 2\alpha - l_{ij}(t_{ij}, \theta))
\item The set $\mathcal{T}_m$, in which the pair weight is positive both in the noise free and the noisy cases:
\begin{equation}
\label{eq:T_m}
\begin{split}
\hat{w}_{ij}\hat{l}(x_i, x_j, t_{ij}; \theta) &= (1 - q_{t_{ij}} - q_{t_{ij}}\frac{\hat{w}_{ij}}{w_{ij}})l_{ij}(t_{ij};\theta) \\
& + 2(d_{max} + 2\alpha)q_{t_{ij}}\hat{w}_{ij}
\end{split}
\end{equation}
\end{compactitem}

Using Eq.~\ref{eq:T_plus}, ~\ref{eq:T_minus}, and~\ref{eq:T_m} one can define the empirical risk under noise over the whole training set $\mathcal{T} = \mathcal{T}_m^+ \bigcup \mathcal{T}_m^- \bigcup \mathcal{T}_m$ for marginal loss $\hat{R}_{l^M}(\mathcal{T},\theta)$.

Then, we consider the local minimizer $\theta^*$ of the noise free loss $R_{l^M}(\mathcal{T},\theta)$ where $||\theta^*-\theta||^2_2 < \epsilon$. Assuming that the locality $\epsilon$ is small enough so that the pair subsets $\mathcal{T}_m^+, \mathcal{T}_m^-, \mathcal{T}_m$ are the same for $\theta$ and $\theta^*$.
Consequently, we can expand $\hat{R}_l(\mathcal{T},\theta^*) -  \hat{R}_l(\mathcal{T},\theta)$ similarly to Eq.~\ref{eq:risk_min3} as:

\begin{equation}
\label{eq:risk_min4}
\begin{split}
\hat{R}_{l^M}&(\mathcal{T},\theta^*) -  \hat{R}_{l^M}(\mathcal{T},\theta) \leq \\
&Q\frac{| \mathcal{T}_m \bigcup \bar{\mathcal{T}}_m^+|}{|\mathcal{T}|}\Big[
\big(R_{l^M}(\mathcal{T},\theta^*) -  R_{l^M}(\mathcal{T},\theta)\big) \Big] \\
+\frac{1}{|\mathcal{T}|}&\Big[
\sum_{i,j\in \mathcal{T}_m^+}(1 - q_{t_{ij}} - Q)w_{ij}\big(l_{ij}(t_{ij}, \theta^*) -  l_{ij}(t_{ij}, \theta)\big) \Big]  \\
+ \frac{1}{|\mathcal{T}|}&\Big[\sum_{i,j\in \bar{\mathcal{T}}_m^-}{q_{t_{ij}}\hat{w}_{ij}\big(l_{ij}(-t_{ij}, \theta^*)-l_{ij}(-t_{ij}, \theta)\big)}\Big]
\end{split}
\end{equation}

Given that $R_{l^M}(\mathcal{T},\theta^*) -  R_{l^M}(\mathcal{T},\theta)) \leq 0$, the condition in which $\hat{R}_{l^M}(\mathcal{T},\theta^*) -  \hat{R}_{l^M}(\mathcal{T},\theta) \leq 0$ is when the sum of the last 2 residual terms in Eq.~\ref{eq:risk_min4} is also negative.
Though the condition on the residual cannot be proven analytically, we observe that:
\begin{compactitem}
\item In practice, as $l_{ij}(-t_{ij}, \theta)$ is bounded, therefore an corrupted pair also contributes only a bounded positive value.
\item Given that the condition $\theta^*$ is sufficiently good, the first residual term will contribute negative to the sum.
\item If $z|\mathcal{T}_{m}^+| > |\mathcal{T}_{m}^-|$ with some value $z$, ie. there are enough correct pairs to counter negative pairs with positive weights, the first correct term of the residual could outweigh the second noisy term, thus assuring that the residual sum is negative.
\end{compactitem}
Further expansion of Eq.~\ref{eq:risk_min4} and the estimation of $z$ are provided in the supplementary to support the observations.

This additional condition on the residual varies based on practical properties of datasets and can be satisfied in practice when the noise rate is small. 
We conjecture that the resistance of marginal loss in a small locality dominantly depends on the 2 prerequisites: $Q \geq 0.5$ and the model being sufficiently good in the ideal clean dataset.

For the condition $Q \geq 0.5$, applying the new values of $w_{ij}$ into Eq.~\ref{eq:min_p}, we can calculate the new bound using the non-zero weights $w_{ij}$ to be:
\begin{equation}
\begin{split}
p &< 1 - \sqrt{1 - \frac{\eta^-}{\eta^+}}  - \epsilon(K) \\
 &= 1 - \sqrt{1 - \gamma} - \epsilon(K)
\end{split}
\end{equation}
with $\eta^+ = max\{\eta^{+}_{ij}\}$, $\eta^- = min\{\eta^{-}_{ij}\}$, $\eta^- \leq \eta^+$, and $\gamma = \frac{\eta^-}{\eta^+}$. To achieve a high bound, we need to $\gamma$ to be close to $1$. However, we cannot tune the values of $\eta^-$ and $\eta^+$. Realistically, a sampling method should choose more diverse positive pairs (minimizing $\eta^+$) as well as sufficiently diverse negative pairs with respect to the positive pairs ($\eta^- \leq \eta^+$).
%To estimate the value of $\eta^+_{max}$ and $\eta^-_{max}$, we can use the fact that in high dimensional space, the distribution of distances between 2 points on the hypersphere can be approximated by a Gaussian $d_{ij}\sim \mathcal{N}(\mu=\sqrt{2},\,\sigma^{2}=\frac{1}{2h})$, with $h$ is the number of dimensions~\cite{manmatha2017sampling,hammersley1950distribution}. Using this fact, we have:
%\begin{equation}
%\begin{split}
%\eta^+ &\geq \frac{1}{1 - \mathcal{P}(d \leq \beta + \alpha)}\\
%\eta^- &\geq \frac{1}{1 - \mathcal{P}(d \geq \beta - \alpha)}\\
%\end{split}
%\end{equation}
%For example, if $\beta = 1.4$ and $\alpha = 0.2$, we can estimate the bound for marginal loss to be tolerant to label noise is when $p < p_{max} < 0.35$.
%
%As the bound in marginal loss is controlled by 2 different weighting values, we can argue that this makes marginal loss less robust to label noise comparing to triplet loss.
%%
%Intuitively, in unhinged pair-wise losses such as auxiliary loss and marginal loss, even though easier pairs contribute small losses for optimization, they are more likely to be correct and can help balancing out the noisy pairs.
%%
%As the hinge function can be consider as a sample filtering function, it filters out more easy pairs in marginal loss than in triplet loss. This makes triplet loss at least equally or more robust than marginal loss against label noise.

\section{Experiments}

\begin{figure*}
\centering
a)
\epsfig{file=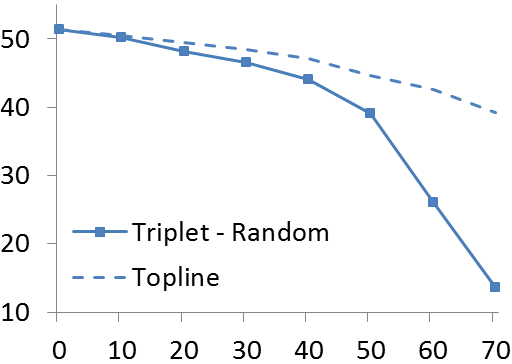,height=25mm}
\quad
b)
\epsfig{file=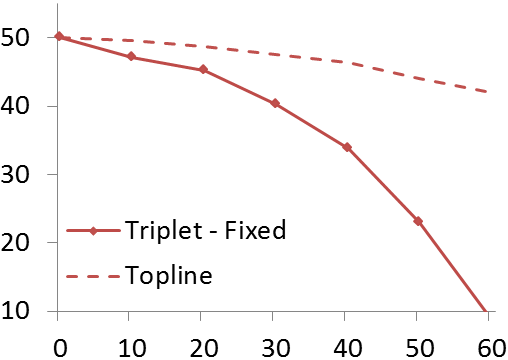,height=25mm}
\quad
c)
\epsfig{file=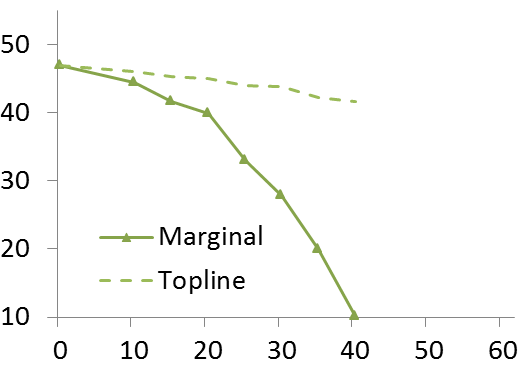,height=25mm}
\quad
d)
\epsfig{file=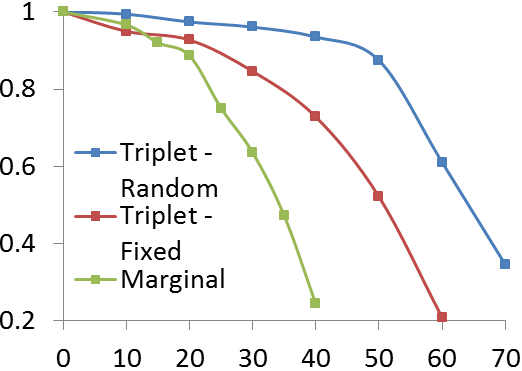,height=25mm}
\caption{Retrieval results reported on Standford Online Products dataset. $x$-axis: noise rate $p$. (a-c)  $y$-axis: Rec@1 of triplet loss with random semi-hard mining, fixed semi-hard mining, and marginal loss with random semi-hard mining, respectively. (d)  $y$-axis: the ratio of Rec@1 for noise rate $p$ over Rec@1 when there are $1-p$ data samples (topline) for all three cases.}
%\vspace{-3mm}
\label{fig:sop_rn34}
\end{figure*}

\begin{figure*}
\centering
a)
\epsfig{file=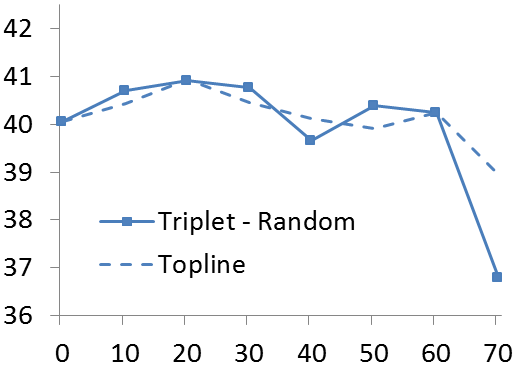,height=25mm}
\quad
b)
\epsfig{file=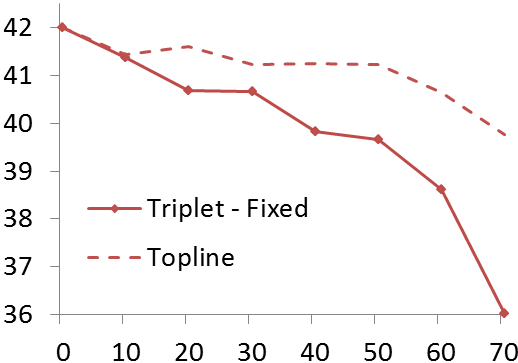,height=25mm}
\quad
c)
\epsfig{file=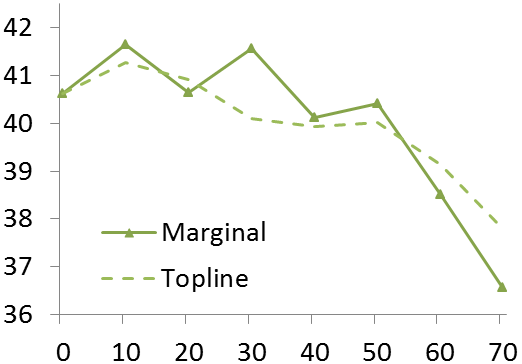,height=25mm}
\quad
d)
\epsfig{file=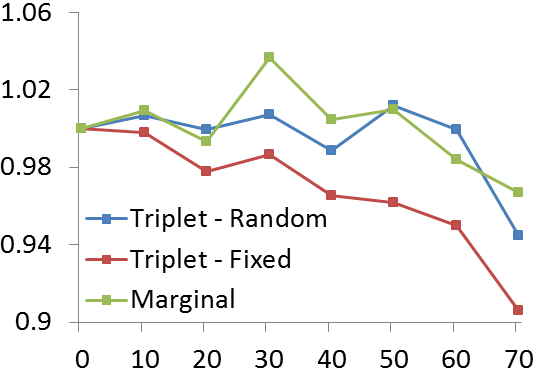,height=25mm}
\caption{Retrieval results reported on CUB-200-2011 birds dataset. $x$-axis: noise rate $p$.(a-c)  $y$-axis: Rec@1 of triplet loss with random semi-hard mining, fixed semi-hard mining, and marginal loss with random semi-hard mining, respectively. (d)  $y$-axis: the ratio of Rec@1 for noise rate $p$ over Rec@1 when there are $1-p$ data samples (topline) for all three cases.}
%\vspace{-3mm}
\label{fig:cub_init_rn34}
\end{figure*}

\subsection{Preliminary settings}

\mypartitle{Datasets.} We illustrate the guarantees through experiments on 3 datasets: Stanford online product (SOP) dataset \cite{song2016deep}, CUB-200-2011 bird dataset~\cite{WahCUB_200_2011}, and Oxford-102 Flowers dataset~\cite{nilsback2008automated}.
%It contains 120k images of 23k classes of online products. 50\% of classes are used for training and the rest for testing.

\mypartitle{Metrics.} For the image retrieval task, we use the Recall@K as in \cite{song2016deep}. For the clustering task, we use the Normalized Mutual Information (NMI) score to evaluate the quality of clustering alignments given a labled groundtruth clustering~\cite{manmatha2017sampling}. We use K-means algorithm for clustering.

\mypartitle{Architecture and training.} We use the ResNet architecture with 34 layers~\cite{he2016deep}. The optimizer is RMSProp~\cite{tieleman2012lecture} and the minibatch size is 60 (12 classes x 5 images). For the CUB and Flowers datasets, we use the pretrained classification model on ImageNet.

\mypartitle{Loss parameters.} For triplet loss, we choose $\alpha = 0.2$. For marginal loss, $\beta=1.4$ and $\alpha=0.2$.

\mypartitle{Reference topline.} As having noisy labels also means there are fewer correct data points for training. Hence, to disentangle the effect of lacking data, we compare the result of learning with noise rate $p$ with the topline result of learning with only $1-p$ clean random data samples.

\subsection{Analysis}
%We use NMI score, $I(\Omega,C) / \sqrt{H(\Omega)H(C)}$, to evaluate the quality of clustering alignments $C = {c_1, . . . , c_n}$, given a groundtruth clustering $Ω = {\omega_1, . . . , \omega_n}$. Here I$(·, ·)$ and $H(·)$ denotes mutual information and entropy respectively. We use K-means algorithm for clustering.

\mypartitle{Triplet loss.} In the image retrieval task, triplet loss is robust to label noise and varies differently based on each dataset and the sampling strategy. When there is no label noise, triplet loss with fixed semi-hard mining performs slightly better than with random semi-hard mining. However, when there is label noise, fixed semi-hard deteriorates faster. In SOP dataset (Fig.~\ref{fig:sop_rn34}-a, b), the gap between learning with noise and learning with fewer clean labels widen significantly after 30\% for fixed semi-hard mining while random semi-hard mining still retains good relative performance after 50\%. The difference can be examined directly by comparing the ratio between accuracy with noise over accuracy with fewer samples in Fig.~\ref{fig:sop_rn34}-d, where fixed semi-hard mining is clearly below random semi-hard mining. The same behaviour is observed in CUB dataset (Fig.~\ref{fig:cub_init_rn34}-a, b, d) and Flowers dataset (Fig.~\ref{fig:oxf_init_rn34}-a, b, d)
This result shows how different sampling strategies affect the robustness to label noise differently. It also corroborates our conjecture that $\eta_{fixed} > \eta_{rand}$.

\mypartitle{Marginal loss.} Compared to triplet loss, marginal loss exhibits a higher variance of robustness across datasets in image retrieval task. In SOP, marginal loss degrades much faster than both versions of triplet loss, as shown in Fig.~\ref{fig:sop_rn34}-c,d. 
Meanwhile in CUB dataset, marginal loss is relatively as robust as triplet loss with random semi-hard mining, with the breakpoint of 50\% comparing to 60\% in triplet loss (Fig.~\ref{fig:cub_init_rn34}-c,d).
In Flowers dataset, even though the performance of marginal loss decreases slightly faster than that of triplet loss, it may due to the fact that marginal loss performs worse with fewer data rather than due to noise (Fig.~\ref{fig:oxf_init_rn34}-c). When comparing the relative measurement, it still shows the same degree of robustness with triplet loss.
To explain the discrepancy across datasets, we consider the fact that the guarantee for marginal loss is only applicable for local minimizer. Because in CUB and Flower datasets, we start with the pretrained model on ImageNet, which means the initial $\theta$ is already good and the final $\theta^*$ is reachable through local optimizing steps.

\begin{figure*}
\centering
a)
\epsfig{file=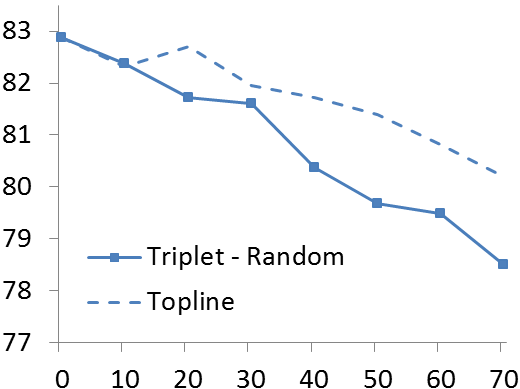,height=25mm}
\quad
b)
\epsfig{file=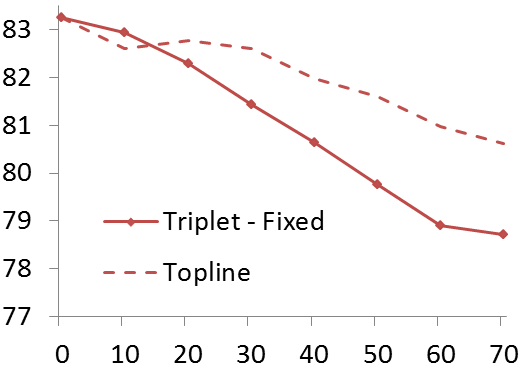,height=25mm}
\quad
c)
\epsfig{file=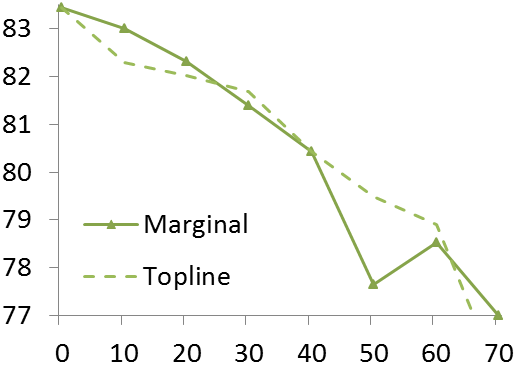,height=25mm}
\quad
d)
\epsfig{file=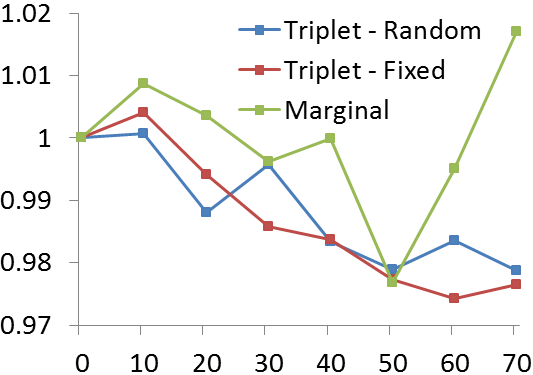,height=25mm}
\caption{Retrieval results reported on Oxford-102 flowers dataset. $x$-axis: noise rate $p$.(a-c)  $y$-axis: Rec@1 of triplet loss with random semi-hard mining, fixed semi-hard mining, and marginal loss with random semi-hard mining, respectively. (d)  $y$-axis: the ratio of Rec@1 for noise rate $p$ over Rec@1 when there are $1-p$ data samples (topline) for all three cases.}
%\vspace{-3mm}
\label{fig:oxf_init_rn34}
\end{figure*}

\mypartitle{Additional results on clustering tasks.} In Fig.~\ref{fig:nmi_all}, we show the ratios of the NMIs under noise rate $p$ over the NMIs of missing rate $p$ of data (topline) for all investigated methods in all 3 datasets. Overall, the results in the clustering task agree with our conclusions from the image retrieval task. Using random semi-hard mining with triplet loss yields more diverse negative pairs, making it more robust to label noise than fixed semi-hard mining. Marginal loss with good initialization shows a statistically similar level of robustness with triplet loss. More detailed figures on the clustering task are provided in the supplementary.

\begin{figure}
\centering
a)
\epsfig{file=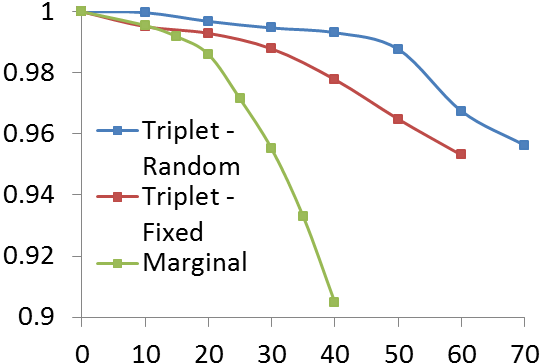,height=25mm}
b)
\epsfig{file=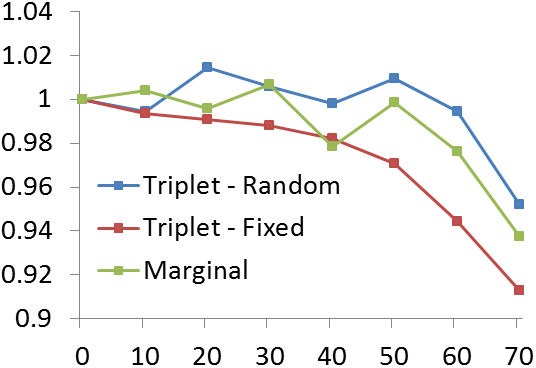,height=25mm}
\\
c)
\epsfig{file=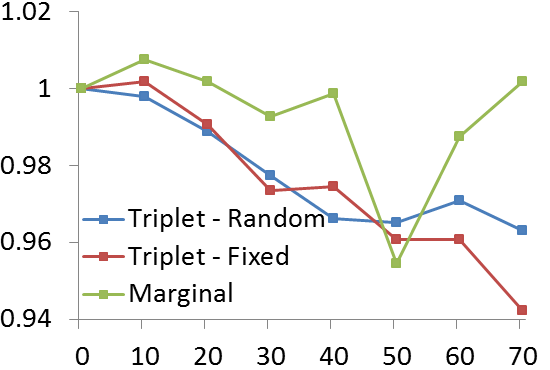,height=25mm}
\caption{Clustering results $x$-axis: noise rate $p$, $y$-axis: the ratio of NMI for noise rate $p$ over NMI when there are $1-p$ data samples (topline) for triplet loss with random semi-hard sampling, fixed semi-hard mining, and marginal loss with random semi-hard sampling. (a-c) results on the SOP, CUB, and Flowers datasets, respectively}
%\vspace{-3mm}
\label{fig:nmi_all}
\end{figure}

\section{Discussion}

A major impact of our theoretical results is in creating effective guidelines to learn embeddings with 
unsupervised and weakly supervised datasets as follows:

\mypartitle{Having high precision in sampling positive pairs.} From Eq.~\ref{eq:min_p}, we can show that the probability a negative pair is corrupted into a positive pair has a dominant impact on the learned embedding. Intuitively, a wrong positive pair is always sampled while a wrong negative pair may not be sampled at all. 
Therefore in practice, when the labels are known to be unreliable, it is better to keep the high precision when choosing positive pairs. This is the reason why the systems proposed in ~\cite{jansen2018unsupervised} worked, as the positive pairs are created by self-transformation or selected with high fidelity. Even with lower precision, we have shown in the experiments that the guarantee can be as high as $50\%$ noise rate for triplet loss with random semi-hard mining. This fact can be used to explain why the unsupervised mining method in~\cite{iscen2018mining} with positive pair noise rate of $q\approx 60\%$ ($p\approx 36.8\%$) can achieve the same or even better accuracy than supervised methods with clean data.

\mypartitle{Using pretrained models.} Because the ideal model should be sufficiently good, we conjecture that the better starting from a good initialization of parameters $\theta$ can help to the increase resistant level. This is also quite intuitive because pretrained models are learned with clean data
Using good initialization is even more significant for marginal loss because of the local condition.
Hence, it is advised to use pretrained models when learning with unreliable labels. In fact, this has been already a standard practice in~\cite{iscen2018mining,lee2017large}.

\section{Conclusion}
We have provided the theoretical guarantees of the 2 common losses for embedding learning: triplet loss and marginal loss. 
Our analysis shows a dependence between the sampling strategies and the resistance against label noise in embedding learning. 
Such guarantees are useful for practical tasks when we want to learn a good embedding (without any change in the algorithm or network architecture) even if the training set labels are noisy.
We demonstrate our results on standard image retrieval datasets. Furthermore, we analyze and provide practical guidelines for future works in unsupervised and weakly supervised learning.
There are several potential research directions to extend our work.
The first one is to investigate other sampling strategies such as in~\cite{manmatha2017sampling,harwood2017smart}.
Other embedding losses, for example quadruple loss~\cite{chen2017beyond}, N-pair loss~\cite{sohn2016improved}, or marginal loss with learnable $\beta$~\cite{manmatha2017sampling}.

%\documentclass[10pt,twocolumn,letterpaper]{article}
%
%\usepackage{cvpr}
%\usepackage{times}
%\usepackage{epsfig}
%\usepackage{graphicx}
%\usepackage{amsmath}
%\usepackage{amsthm}
%\usepackage{amssymb}
%
%% Include other packages here, before hyperref.
%
%% If you comment hyperref and then uncomment it, you should delete
%% egpaper.aux before re-running latex.  (Or just hit 'q' on the first latex
%% run, let it finish, and you should be clear).
%\usepackage[pagebackref=true,breaklinks=true,letterpaper=true,colorlinks,bookmarks=false]{hyperref}
%
%%\cvprfinalcopy % *** Uncomment this line for the final submission
%
%\def\cvprPaperID{5719} % *** Enter the CVPR Paper ID here
%\def\httilde{\mbox{\tt\raisebox{-.5ex}{\symbol{126}}}}
%
%\usepackage{hyperref}
%\usepackage{url}
%\usepackage{amsmath}
%\usepackage{graphicx}
%\usepackage{paralist}
%\usepackage{amssymb}
%\usepackage{mathtools}
%\DeclareMathOperator*{\argmin}{arg\,min}
%\DeclareMathOperator*{\rand}{rand}
%\newtheorem{mydef}{Definition}
%\newtheorem{myconj}{Conjecture}
%\newtheorem{proposition}{Proposition}
%\newtheorem{lemma}{Lemma}

\renewcommand{\labelenumii}{\theenumii}
\renewcommand{\theenumii}{\theenumi.\arabic{enumii}.}

%% Pages are numbered in submission mode, and unnumbered in camera-ready
%\ifcvprfinal\pagestyle{empty}\fi
%\input{Notations}
%\graphicspath{{IMAGES/}}
%
%\begin{document}
%
%\title{Supplementary Material:\\
%Theoretical Guarantees of Deep Embedding Losses Under Label Noise}
%
%\author{
%Nam Le$^{1,2}$ \qquad Jean-Marc Odobez$^{1,2}$ \\
%$^1$ Idiap Research Institute, Martigny, Switzerland\\
%$^2$ \'{E}cole Polytechnique F\'{e}d\'{e}ral de Lausanne, Switzerland \\
%{\tt\small \{nle,odobez\}@idiap.ch}
%}
%
%
%\maketitle
%
%
%This document provides the supplementary material for our submission ID 5719. It provides further expansion of 2 main equations and more detailed figures for the clustering experiment. All references to Equations and Sections are from the original paper.

\renewcommand\thesection{\Alph{section}}
\setcounter{section}{0}

\section{Supplementary Material} 

\subsection{Upper bound in Equation 21 - Section 4.1}

From Equation 19 - Section 4.1, we have:
\begin{equation}
\begin{split}
\hat{R}_{l^U}(\theta^*&) - \hat{R}_{l^U}(\theta) \\
=\frac{1}{Z}&\sum_{ij}\Big[
\Big(1-q_{t_{ij}} - q_{t_{ij}}\frac{w_{-t_{ij}}}{w_{t_{ij}}}\Big) \\
&\times\Big(w_{t_{ij}}l^A(x_i, x_j, t_{ij}, \theta^*) - w_{t_{ij}}l^A(x_i, x_j, t_{ij}, \theta)\Big)\Big]
\end{split}
\end{equation}

The set of pairs $(i,j)$ can be divided into the positive pairs, $t_{ij}=1$, and negative pairs, $t_{ij}=-1$. Hence the empirical risk difference is also split into:
\begin{equation}
\label{eq:temp_1}
\begin{split}
\hat{R}_{l^U}(\theta^*&) - \hat{R}_{l^U}(\theta) \\
=\frac{1}{Z}\Big[&\sum_{ij/t_{ij}=1} \Big(1-q_{+1} - q_{+1}\frac{w_{-1}}{w_{+1}}\Big)w_{t_{ij}} \\
&\times\Big(l^A(x_i, x_j, t_{ij}, \theta^*) - l^A(x_i, x_j, t_{ij}, \theta)\Big) \\
&+\sum_{ij/t_{ij}=-1} \Big(1-q_{-1} - q_{-1}\frac{w_{+1}}{w_{-1}}\Big)w_{t_{ij}} \\
&\times\Big(l^A(x_i, x_j, t_{ij}, \theta^*) - l^A(x_i, x_j, t_{ij}, \theta)\Big)\Big] \\
=\frac{1}{Z}\Big[&\Big(1-q_{+1} - q_{+1}\frac{w_{-1}}{w_{+1}}\Big)w_{+1} \\
&\times\sum_{ij/t_{ij}=1}\Big(l^A(x_i, x_j, t_{ij}, \theta^*) - l^A(x_i, x_j, t_{ij}, \theta)\Big) \\
&+\Big(1-q_{-1} - q_{-1}\frac{w_{+1}}{w_{-1}}\Big)w_{-1} \\
&\times\sum_{ij/t_{ij}=-1}\Big(l^A(x_i, x_j, t_{ij}, \theta^*) - l^A(x_i, x_j, t_{ij}, \theta)\Big)\Big] \\
\end{split}
\end{equation}

We define the 2 summations as $\mathcal{S}^+$ and $\mathcal{S}^-$:
\begin{equation}
\begin{split}
\mathcal{S}^+ = \sum_{ij/t_{ij} = 1}\Big(l^A(x_i, x_j, t_{ij}, \theta^*) &- l^A(x_i, x_j, t_{ij}, \theta)\Big)\\
\mathcal{S}^- = \sum_{ij/t_{ij} = -1}\Big(l^A(x_i, x_j, t_{ij}, \theta^*) &- l^A(x_i, x_j, t_{ij}, \theta)\Big)
\end{split}
\end{equation}
which will simplify the empirical risk difference into:
\begin{equation}
\begin{split}
\hat{R}_{l^U}(\theta^*) - \hat{R}_{l^U}(\theta) = \frac{1}{Z}\Big[&\Big(1-q_{+1} - q_{+1}\frac{w_{-1}}{w_{+1}}\Big)w_{+1} \mathcal{S}^+ \\
&+\Big(1-q_{-1} - q_{-1}\frac{w_{+1}}{w_{-1}}\Big)w_{-1} \mathcal{S}^- \\
\end{split}
\end{equation}

\noindent From the condition, we have $\theta^*$ is also the minimizer of $\mathcal{S}^+$ and $\mathcal{S}^-$, or:
\begin{equation}
\begin{split}
\mathcal{S}^+ \leq 0 \quad \text{and} \quad \mathcal{S}^- \leq 0
\end{split}
\end{equation}

Because $\mathcal{S}^+$ and $\mathcal{S}^-$ are negative, the smaller the multiplier is, the bigger the value of the empirical risk difference. By choosing the multiplier to be the minimum value of $1-q_{t_{ij}} - q_{t_{ij}}\frac{w_{-t_{ij}}}{w_{t_{ij}}}$, we achieve the upper bound as in Equation 21 - Section 4.1:

\begin{equation}
\begin{split}
\hat{R}_{l^U}(S, \theta^*) - \hat{R}_{l^U}(S, \theta&) \\
\leq \min_{t_{ij}}\Big(1-q_{t_{ij}} &- q_{t_{ij}}\frac{w_{-t_{ij}}}{w_{t_{ij}}}\Big)(\mathcal{S}^+ + \mathcal{S}^-) \\
= \min_{t_{ij}}\Big(1-q_{t_{ij}} &- q_{t_{ij}}\frac{w_{-t_{ij}}}{w_{t_{ij}}}\Big)\\
\times\sum_{ij}\Big(l^A(&x_i, x_j, t_{ij}, \theta^*) - l^A(x_i, x_j, t_{ij}, \theta)\Big) \\
= \min_{t_{ij}}\Big(1-q_{t_{ij}} - &q_{t_{ij}}\frac{w_{-t_{ij}}}{w_{t_{ij}}}\Big)\Big(R_{l^U}(S, \theta^*) - R_{l^U}(S, \theta)\Big)
\end{split}
\end{equation}

\subsection{Expansion of Equation 36 - Section 5.3}

Let the residual term in Equation 35 - Section 5.3 be $r$, which is:
\begin{equation}
\begin{split}
r = \frac{1}{|\mathcal{T}|}&\Big[
\sum_{i,j\in \mathcal{T}_m^+}(1 - q_{t_{ij}} - Q)w_{ij}\big(l_{ij}(t_{ij}, \theta^*) -  l_{ij}(t_{ij}, \theta)\big) \Big]  \\
+ \frac{1}{|\mathcal{T}|}&\Big[\sum_{i,j\in \bar{\mathcal{T}}_m^-}{q_{t_{ij}}\hat{w}_{ij}\big(l_{ij}(-t_{ij}, \theta^*)-l_{ij}(-t_{ij}, \theta)\big)}\Big]
\end{split}
\end{equation}

Here, we want to find the condition for $r$ to be negative, or when the first term outweighs the second term. Because the difference in pair-wise loss, ie. $l_{ij}(t_{ij}, \theta^*) -  l_{ij}(t_{ij}, \theta)$, is bounded, we only need to consider when the first multiplier $(1 - q_{t_{ij}} - Q)w_{ij}$ is bigger than the second multiplier $q_{t_{ij}}\hat{w}_{ij}$. 

To this end, we first need to compute the value of $Q$. As $K$ is assumed to be very big, we can approximate the multipliers $1-q_{t_{ij}} - q_{t_{ij}}\frac{w_{-t_{ij}}}{w_{t_{ij}}}$ for the positive and negative cases as:
\begin{equation}
\begin{split}
1-q_{+1} - q_{+1}\frac{w_{-1}}{w_{+1}} &\approx 1 - 2p + p^2 - (2p - p^2)\frac{\eta^-_{ij}}{K\eta^+_{ij}} \\
1-q_{-1} - q_{-1}\frac{w_{+1}}{w_{-1}} &\approx 1 - \frac{2p - p^2}{K} - (2p - p^2)\frac{\eta^+_{ij}}{\eta^-_{ij}}
\end{split}
\end{equation}

By setting $\eta^+ = max\{\eta^{+}_{ij}\}$, $\eta^- = min\{\eta^{-}_{ij}\}$ and assuming that $\eta^- \leq \eta^+$,
%Because the multiplier of a negative pair is always smaller due to $\eta^+_{ij}\eta^-_{ij} > 1$, 
we can approximate the minimum value of $Q$ as:
\begin{equation}
\begin{split}
Q &= \min_{t_{ij}}\Big(1-q_{t_{ij}} - q_{t_{ij}}\frac{w_{-t_{ij}}}{w_{t_{ij}}}\Big) \\
&\approx 1 - (2p - p^2)\frac{\eta^+}{\eta^-}
\end{split}
\end{equation}

Consider the set $\mathcal{T}_m^+$, the multiplier for each label $-1$ or $1$ of one pair is:
\begin{equation}
\begin{split}
(1 - q_{+1} - Q)w_{ij} &\approx (2p - p^2)(\frac{\eta^+}{\eta^-} - 1)\eta^+_{ij} \\
(1 - q_{-1} - Q)w_{ij} &\approx (2p - p^2)\frac{\eta^+}{\eta^-}\frac{\eta^-_{ij}}{K}
\end{split}
\end{equation}

Consider the set $\mathcal{T}_m^-$, the multiplier for each label $-1$ or $1$ of one pair is:
\begin{equation}
\begin{split}
q_{+1}\hat{w}_{ij} &\approx (2p - p^2)\frac{\eta^-_{ij}}{K}\\
q_{-1}\hat{w}_{ij} &\approx \frac{2p - p^2}{K}\eta^+_{ij}
\end{split}
\end{equation}
By assuming that $\eta^+_{ij}$ and $\eta^-_{ij}$ are bounded, we can choose $z = \frac{1}{\frac{\eta^+}{\eta^-} - 1}$. With high probability, we can have:
\[z(1 - q_{t_{ij}} - Q)w_{ij} \geq q_{t_{ij}}\hat{w}_{ij}\]
Therefore, the residual is negative with high probability when  $z|\mathcal{T}_m^+| > |\mathcal{T}_m^-|$. This means that even though harder pairs are more likely to be error, as long as we have sufficiently many good pairs to counter-balance it, the local optimization of empirical risk with label noise will still yields the same local minimizer.

\subsection{Clustering experiment}

\begin{figure}
\centering
a)
\epsfig{file=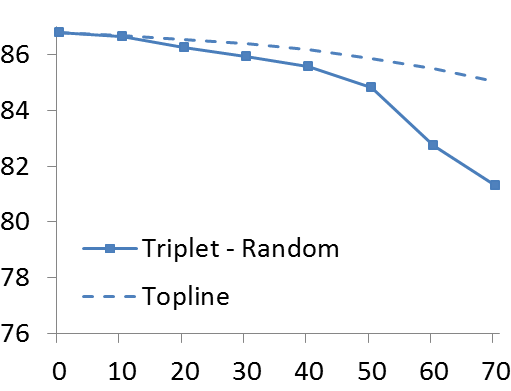,height=25mm}
b)
\epsfig{file=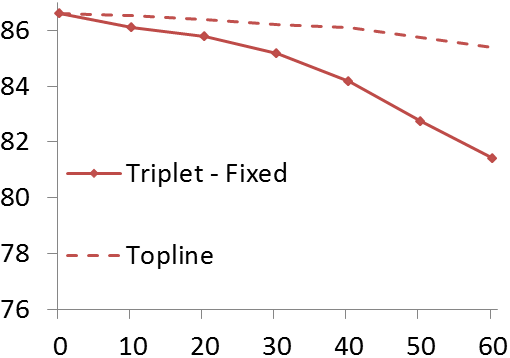,height=25mm}
\quad
c)
\epsfig{file=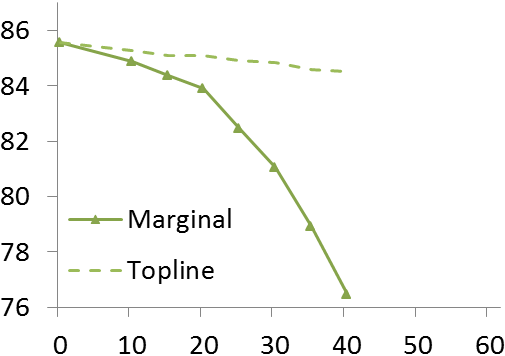,height=25mm}
d)
\epsfig{file=nmi-sop-all.png,height=25mm}
\caption{Clustering results reported on Standford Online Products dataset. $x$-axis: noise rate $p$, $y$-axis: NMI.(a-c) NMI of triplet loss with random semi-hard mining, fixed semi-hard mining, and marginal loss with random semi-hard mining, respectively. (d) the ratio of NMI for noise rate $p$ over NMI when there are $1-p$ data samples (topline) for all three cases.}
%\vspace{-3mm}
\label{fig:nmi_sop_rn34}
\end{figure}

\begin{figure}
\centering
a)
\epsfig{file=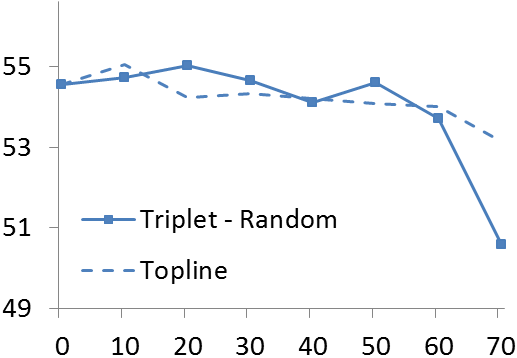,height=25mm}
b)
\epsfig{file=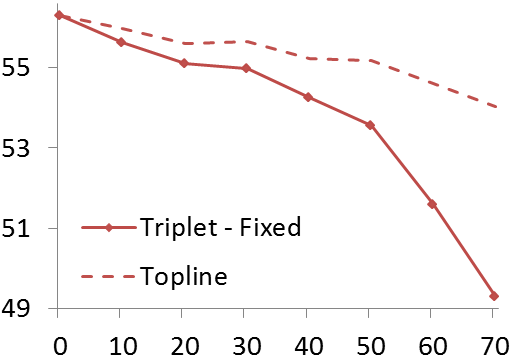,height=25mm}
\quad
c)
\epsfig{file=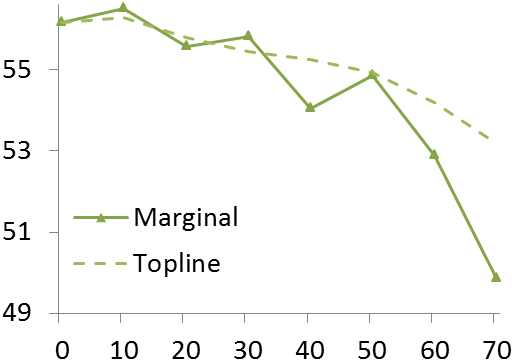,height=25mm}
d)
\epsfig{file=nmi-cub-all.png,height=25mm}
\caption{Clustering results reported on CUB-200-2011 birds dataset. $x$-axis: noise rate $p$, $y$-axis: NMI.(a-c) NMI of triplet loss with random semi-hard mining, fixed semi-hard mining, and marginal loss with random semi-hard mining, respectively. (d) the ratio of NMI for noise rate $p$ over NMI when there are $1-p$ data samples (topline) for all three cases.}
%\vspace{-3mm}
\label{fig:nmi_cub_init_rn34}
\end{figure}

\begin{figure}
\centering
a)
\epsfig{file=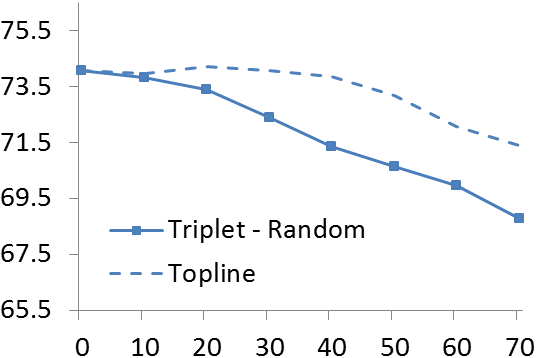,height=25mm}
b)
\epsfig{file=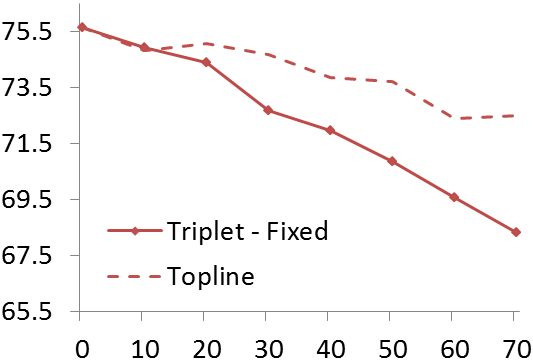,height=25mm}
\quad
c)
\epsfig{file=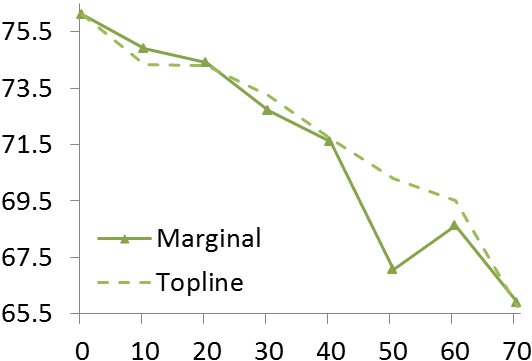,height=25mm}
d)
\epsfig{file=nmi-flo-all.png,height=25mm}
\caption{Clustering results reported on Oxford-102 flowers dataset. $x$-axis: noise rate $p$, $y$-axis: NMI.(a-c) NMI of triplet loss with random semi-hard mining, fixed semi-hard mining, and marginal loss with random semi-hard mining, respectively. (d) the ratio of NMI for noise rate $p$ over NMI when there are $1-p$ data samples (topline) for all three cases.}
%\vspace{-3mm}
\label{fig:nmi_oxf_init_rn34}
\end{figure}

In the clustering tasks, we use the Normalized Mutual Information (NMI) metrics to quantify the clustering quality.
$NMI = I(\Omega,C) / \sqrt{H(\Omega)H(C)}$, with $C = {c_1, . . . , c_n}$ being the clustering alignments, and $\Omega = {\omega_1, . . . , \omega_n}$ being the given groundtruth clusters (ie. class labels). Here $I(·, ·)$ and $H(·)$ denotes mutual information and entropy respectively. We use K-means algorithm for clustering.

Because measuring clustering quality takes into account all nearby neighbors instead of just the nearest one, the difference in NMI between methods are narrower than in Rec@1.
Still, the results in the clustering task agree with our conclusions from the image retrieval task. Using random semi-hard mining with triplet loss is more robust to label noise than fixed semi-hard mining and good minimziation helps to make marginal loss more robust to label noise. 

In SOP dataset (Fig. \ref{fig:nmi_sop_rn34}), the deterioration of triplet loss with fixed semi-hard mining increases after 40\% while random semi-hard mining still retains good relative performance after 50\%. Meanwhile marginal loss degrades much faster than both versions of triplet loss. The difference is easier to view when we compare the ratios between the NMI under noise and the NMI with few data.

In CUB and Flower datasets (Fig.~\ref{fig:nmi_cub_init_rn34} and Fig.~\ref{fig:nmi_oxf_init_rn34}), we observe the same difference between triplet loss with fixed or random semi-hard mining. On the other hand,marginal loss is relatively as robust as triplet loss with random semi-hard mining. This fact, as shown in the paper, can be contributed by initialization with pretrained models.

%\end{document}

\bibliographystyle{abbrv}
\bibliography{Main}
\end{document}